\documentclass{article}

\usepackage{microtype}
\usepackage{graphicx}
\usepackage{subfigure}
\usepackage{booktabs} 

\usepackage{hyperref}

\usepackage[accepted]{icml2018}

\usepackage{graphics} 
\usepackage{epsfig} 
\usepackage{amsmath} 
\usepackage{amssymb}  
\usepackage{multirow}
\usepackage{booktabs}
\usepackage{pgfplotstable}
\usepackage{pgfplots}
\usepackage{graphicx}
\usepackage{tikz}
\usepackage{todonotes}

\usepackage{amsthm}

\usepackage{lipsum}

\newtheorem{theo}{Theorem}
\newtheorem{prop}{Proposition}

\begin{document}

\twocolumn[
\icmltitle{Residual Networks: Lyapunov Stability and Convex Decomposition}




\begin{icmlauthorlist}
\icmlauthor{Kamil Nar}{ucb}
\icmlauthor{Shankar Sastry}{ucb}
\end{icmlauthorlist}

\icmlaffiliation{ucb}{Electrical Engineering and Computer Sciences, \hbox{University} of California, Berkeley}

\icmlcorrespondingauthor{Kamil Nar}{nar@berkeley.edu}

\icmlkeywords{Machine Learning, ICML}

\vskip 0.3in
]



 \printAffiliationsAndNotice{}  

\begin{abstract}
While training error of most deep neural networks degrades as the depth of the network increases, residual networks appear to be an exception.~We show that the main reason for this is the Lyapunov stability of the gradient descent algorithm:~for an arbitrarily chosen step size, the equilibria of the gradient descent are most likely to remain stable for the parametrization of residual networks.~We then present an architecture with a pair of residual networks to approximate a large class of functions by decomposing them into a convex and a concave part.~
Some parameters of this model are shown to change little during training, and this imperfect optimization prevents overfitting the data and leads to solutions with small Lipschitz constants, while providing clues about the generalization of other deep networks. 
\end{abstract}

\section{Introduction}

For most neural network architectures, the expressiveness of the network improves as the depth increases. However, training and  test errors of most networks have been shown to deteriorate in practice if the depth exceeds a few layers \cite{originalresnet}. This discrepancy indicates a problem with the method used to train these networks. 
%
Given that gradient-based iterative algorithms are the almost exclusive choice for training, 
the most likely cause is the poor \emph{stability} of the dynamical systems created by these algorithms \emph{in the sense of Lyapunov}\footnotemark[2]. This is further substantiated by the sharp falls observed in the training error
 if a time-varying step size is used \cite{originalresnet}.

It is known that the sensitivity of the loss function with~respect to different parameters can become disproportionate while training a neural network. For example, some~of~the gradients might vanish while others explode, and this prevents an effective training.~As a remedy, changing the geometry of the optimization was suggested  and a regularized descent algorithm was introduced in \cite{Neyshabur17}.~This algorithm was shown to converge in fewer iterations over the data, but it required the computation of a scaling constant for each node of the network at each time step, which depended on almost all other nodes if the network is fully connected.

Residual networks appear to be an exception: 
their training error improves with depth even when a standard gradient method is used \cite{originalresnet}.
To explain their different behavior, 
linear versions of these networks have been shown to possess some crucial properties for optimization.~In particular, 
it was shown that all local optima of linear residual networks are also the global optima, and the gradient of their cost function does not vanish away from the local optima \cite{Moritz}.~Later, equivalent results were derived under some conditions for nonlinear residual networks as well \cite{Bartlett1}.

\footnotetext[2]{Lyapunov stability is a property of the equilibria, but for ease of reading, we will refer to the dynamical systems and the algorithms as stable if their equilibria are stable.}

Lyapunov stability was used in the past to understand and improve the training of neural networks \cite{Michel88, Matsuoka92, Man06}, but the success and~the problems of the state-of-the-art networks have not been analyzed from this perspective. In this paper, we fill this~gap and show that the residual networks are indeed the right architecture to be trained by gradient-based algorithms in terms of Lyapunov stability. More precisely, we show that given an arbitrary step size, the equilibria of the gradient descent algorithm are most likely to remain stable in the sense of Lyapunov for residual networks. We also reveal that the equilibria of the gradient descent algorithm could be unstable for most deep neural networks, in which case the algorithm might approach an optimum but not converge to it \emph{even if the algorithm is not stochastic}, thereby providing some level of regularization.
Our result is~fundamentally different from the previous works \cite{Saxe13, Gunasekar} in that we address whether the local optima can actually be achieved by gradient-based methods rather than how well the local optima are.

We then introduce an architecture with a pair of residual networks that can be used to approximate a large class of functions by decomposing them into a convex and a concave part. 
This decomposition elucidates how each layer improves the approximation and provides an interpretable model. We analyze the properties of its local optima and show that the bias parameters are likely to change little during training. Though this seems like a problem, it in fact prevents overfitting and leads to solutions with low Lipschitz constants, which is also associated with generalization. These claims are verified by testing the suggested model on the MNIST data set with no explicit regularization.

The organization of the rest of the paper is as follows. The stability analysis is given in Section 2 and the model for decomposition is introduced in Section 3.~The properties of the model are derived in Section 4 and the results of the test on the MNIST data set are given in Section 5. Lastly, the results are discussed and some future directions are provided in Section 6.

%
%
%

\section{Stability in the Sense of Lyapunov}
Given a discrete time system with state $w[k] \in \mathbb R^n$ at time $k$ and the update rule $\mathcal F: \mathbb R^n \to \mathbb R^n$,
 a point $w^*$ is called an equilibrium of the system 
 $$ w[k+1] = \mathcal F(w[k])$$ 
  if it satisfies  $w^* = \mathcal F(w^*).$ 
The equilibrium point $w^*$ is said to be stable in the sense of Lyapunov if for every $\epsilon > 0$, there exists some $\delta > 0$ such that $\|w[0] - w^*\| < \delta$ implies $\|w[k] - w^*\| < \epsilon$ for all $k \in \mathbb N$. That is, if the state of the system is close to the equilibrium initially, it always stays close to the equilibrium, as shown in Figure 
1. If, in addition, the state converges to the equilibrium, $w^*$ is said to be asymptotically stable in the sense of Lyapunov. 

\begin{figure}[h]
\label{lyapunov}
\centering
\begin{tikzpicture}[scale=0.9]
\draw[dashed] (3,3) circle [radius=1];
\draw[dashed] (3,3) circle [radius=2];
\draw[<-, dashed] (2,3)--(3,3);
\node at (2.4,2.8) {$\delta$};
\node at (1.65,3.85) {$\epsilon$};
\draw[<-, dashed] (1.2,3.9)--(3,3);
\node at  (2.5,3.7) [above] {$w[0]$};
\node at  (3.1,2.85) {$w^*$};
\draw[->,thick] (2.5,3.75) -- (3.2,3.7);
\draw[->,thick] (3.2,3.7) -- (3.8,4.1);
\draw[->,thick] (3.8,4.1) -- (4.1,3.3);
\draw[->,thick] (4.1,3.3) -- (3.8,2.5);
\draw[->,thick] (3.8,2.5) -- (2.9,2.1);
\draw[->,thick] (2.9,2.1) -- (2.4,1.7);
\draw[->,thick] (2.4,1.7) -- (1.4,2.3);
\draw[->,thick] (1.4,2.3) -- (1.7,3.2);
\draw[->,thick] (1.7,3.2) -- (2.4,3.5);
\draw[->,thick] (2.4,3.5) -- (3.0,3.3);
\end{tikzpicture}
\caption{An equilibrium stable in the sense of Lyapunov}
\end{figure}
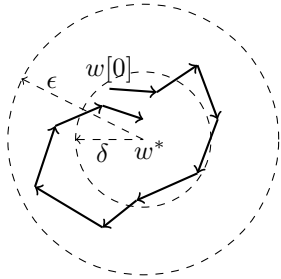

As an example, consider the problem of finding an optimal solution to the ordinary least squares problem
\[ \min_{w \in \mathbb R^n} {1 \over 2} \|X^\top w - y \|^2 \]
via gradient descent, where $X$ and $y$ represent the input and the labels, respectively. Given the step size $\delta$, the update rule for $w[k]$ creates the dynamical system
\begin{eqnarray*} w[k+1] 
&  = & \left( I - \delta XX^\top \right) w[k] + \delta Xy. \end{eqnarray*}
Every equilibrium of this system is asymptotically stable and the gradient descent converges to any of them only if all eigenvalues of the matrix $(I - \delta XX^\top)$ have magnitude less than 1.

As seen from this example, there exists a critical value for the step size above which the gradient descent algorithm becomes unstable and the iterations cannot converge to a local optimum. While this critical value is the same for all equilibria of a linear system, it varies if the dynamical system is nonlinear. Since the gradient descent for deep networks leads to nonlinear dynamics as well, some of its equilibria might become unstable while others are still stable for the same step size, as the following proposition shows.

\begin{prop} Let $f: \mathbb R \to \mathbb R$ be a nonzero linear function, i.e., 
$f(x) = \lambda x$ for all $x \in \mathbb R$ 
and  $\lambda \neq 0$. Given a set of points $\{x_i\}_{i \in [N]}$, assume that $\lambda$ is estimated as a multiplication of the scalar parameters $\{w_j\}_{j \in [L]}$ by minimizing
%
\[ 
{1\over 2N} \sum_{i=1}^N \left(w_L \dots w_2w_1 x_i - f(x_i) \right)^2  \]
via gradient descent.~Then an equilibrium $\{w_j^*\}_{j \in [L]}$ with $w_L^* \dots w_2^* w_1^* = \lambda$ is asymptotically stable only if the step size $\delta$ satisfies
\begin{equation}
\label{bound1}
 \delta \le {2 \over \sigma \sum_{i=1}^L \prod_{j \neq i} (w_j^*)^2} \end{equation}
where $\sigma = \sum_{i=1}^N {x_i^2}/N$. \end{prop}

\begin{proof} The update rule for $w_i[k]$ is given as
\begin{equation}
\label{erk1} w_i[k+1] = w_i[k] - \delta  \sigma \Big( \prod_{j=1}^L w_j[k] - \lambda \Big) \prod_{j \neq i} w_j[k]. \end{equation}
By multiplying these update equations for all $i \in [L]$ and denoting $\left(\prod w_i[k] - \lambda\right)$ by $e[k]$, we obtain
\[ e[k+1]  = e[k] - \delta \sigma e[k] \sum_{i=1}^L \prod_{j \neq i} w_j^2[k]  + o(e[k])\]
where $o(\cdot)$ denotes the higher order terms in its argument. By Lyapunov's indirect method of stability \cite{Khalil}, an equilibrium of a nonlinear system is asymptotically stable only if the linear approximation around that equilibrium is not unstable. Therefore, $e[k]$ converges to zero only if
\[ \Big| 1 - \delta \sigma \sum_{i=1}^L \prod_{j \neq i}(w_j^*)^2 \Big|  \le 1 \]
for the equilibrium $\{w_i^*\}_{i \in [L]}.$ \end{proof}

Note that the bound in (\ref{bound1}) can become very small for an equilibrium with disproportionate parameters such as $\{\kappa \lambda^{1/L}, \dots, \kappa \lambda^{1/L}, \kappa^{1-L} \lambda^{1/L} \}$ with $\kappa \gg 1$; and it attains its largest value
\[ \delta_\text{max} = {2 \over \sigma L \lambda^{2(L-1)/L}}\]
 for the equilibria which satisfy $|w_i^*| = |\lambda|^{1/L}$ for all $i \in [L]$. 
If the step size is even larger than $\delta_\text{max}$, then the gradient descent algorithm cannot converge to any of the  local optima.


Proposition 1 shows that for an arbitrarily chosen $\delta$, the equilibria with $|w_i^*| = |\lambda|^{1/L}$ are most likely to be stable. This suggests, for example, that if $\lambda$ is known to be positive and $L$ is very large, setting $w_i[0] = 1$ for all $i \in [L]$ is a good choice of initialization. In fact, the gradient descent converges exponentially fast with this initialization if $\delta$ is chosen appropriately, which is shown next. 

\begin{prop} Assume that $\lambda > 0$ and $w_i[0] = 1$ for all $i \in [L]$. If the step size $\delta$ is chosen to be less than or equal~to 
\[ \delta_c = \left\{ \begin{array}{l l}
{(L \sigma)^{-1} \lambda^{-2(L-1)/L}} & \text{ if } \lambda \in [1, \infty),\\
\sigma^{-1}(1-\lambda)^{-1}(1 - \lambda^{1/L}) & \text{ if } \lambda \in (0,1),
\end{array} \right. \]
then $|w_i[k] - \lambda^{1\over L}| \le \rho(\delta)^k |1 - \lambda^{1\over L}|$ for all $i \in [L]$, where
\[ \rho(\delta) = \left\{ \begin{array}{l l}
 1 - \delta \sigma (\lambda - 1) (\lambda^{1/L} -1)^{-1}  & \text{ if } \lambda \in (1, \infty),\\
 1 - \delta \sigma L \lambda^{2(L-1)/L}   & \text{ if } \lambda \in (0, 1].
\end{array}
\right.
\]
\end{prop}
\begin{proof} Due to symmetry, $w_i[k] = w_j[k]$ for all $k \in \mathbb N$ for all $i,j \in [L]$. Denoting any of them by $w[k]$, we have
\[ w[k+1] = w[k] - \delta \sigma w^{L-1}[k](w^L[k] - \lambda).\]
To show that $w[k]$ converges to $\lambda^{1/ L}$, we can write
\[ w[k+1]-\lambda^{1/L} = \mu(w[k])(w[k]-\lambda^{1/ L}),\]
where
\vspace{-0.1in}
\[ \mu(w) =  1 - \delta \sigma w^{L-1} \sum_{j=0}^{L-1}w^j\lambda^{(L-1-j)/L}. \] 
\vspace{-0.1in}
If there exists some $\rho \in [0,1)$ such that
\begin{equation}
\label{cond1}
0 \le \mu(w[k]) \le \rho \text{ for all } k \in \mathbb N,
\end{equation}
then $w[k]$ is always larger or always smaller than $\lambda^{1 / L}$, and its distance to $\lambda^{1/ L}$ decreases by a factor of $\rho$ at each step. Since $\mu(w)$ is a monotonic function in $w$, the condition (\ref{cond1}) holds for all $k$ if it holds only for $w[0]=1$ and $\lambda^{1 /L}$, which gives us $\delta_c$ and $\rho(\delta).$ \end{proof}
The identical result also holds for $\lambda < 0$ if $|w_i[0]|=1$ for all $i \in [L]$ and the cardinality of the set $\{i\in [L] : w_i[0] = -1\}$ is odd. However, without knowing the sign of $\lambda$, we cannot decide whether to include a $-1$ in the initialization.~We introduce a decomposition for linear functions to handle this problem in {{Section~3.2.}}

We can extend the results in Propositions 1 and 2 to higher dimensions as well. Assume that $x_i \in \mathbb R^n$ for all $i \in [N]$, and let $R \in \mathbb R^{n \times n}$ be a nonzero matrix. If gradient descent is used with step size $\delta$ to solve
\begin{equation} \label{probb2}
 \min_{W_1, \dots, W_L} {1 \over 2N} \sum_{i=1}^N {\|W_L \dots W_2W_1x_i - Rx_i \|}_2^2, \end{equation}
where $W_i \in \mathbb R^{n \times n}$ for all $i \in [L]$, then the update rule for $W_i$ is 
\begin{equation}
\label{update_eqn}
 W_i[k+1] = W_i[k] - \delta\left( G_i[k](\hat F[k] - R) \Sigma H_i[k] \right), \end{equation}
where $G_i[k] = W_{i+1}^\top[k] \cdots W_L^\top[k]$ for $i < L$, $G_L[k] = I$,
 $\hat F[k] = W_L[k]\cdots W_1[k]$, $H_i[k] = W_1^\top[k] \cdots W_{i-1}^\top[k]$ for $i>1$, $H_1[k]=I$, and
$\Sigma = {1\over N} \sum_{i=1}^N x_ix_i^\top$.

Even though the evolution of each element of $W_i[k]$ seems to depend on all the entries of the other matrices, the system described by (\ref{update_eqn}) decomposes into $n$ independent systems under the assumptions given in Theorem 1. 

%
%


\begin{theo} Let $\lambda(R)$ denote the spectral radius of the~matrix $R$ in (\ref{probb2}). Assume that $R$ is diagonalizable with real eigenvalues, the initial matrices $W_1[0], \dots, W_L[0]$ and $R$ have the identical eigenspaces, and the matrix $\Sigma = I$. If the step size $\delta$ satisfies
\[ \delta > {2 \over L\lambda(R)^{2(L-1)/L}},\]
then none of the equilibria that satisfy $W_L\cdots W_1 = R$ is stable, and the gradient descent cannot converge.
\end{theo}
\begin{proof} There exists a common invertible matrix $M \in \mathbb R^{n \times n}$ that can diagonalize all  matrices in (\ref{update_eqn}): $R = M\Lambda_RM^{-1}$, $W_i = M\Lambda_{W_i}M^{-1}$ for all $i \in [L]$. Then the update rule~(\ref{update_eqn}) turns into $n$ independent update rules for the diagonal elements of $\Lambda_R$ and $\{\Lambda_{W_i}\}_{i \in [L]}$. By Proposition 1, all of these systems can converge only if $$\delta \le { 2 \over  L \lambda_r^{2(L-1)/L}}$$ for each eigenvalue $\lambda_r$ of $R.$
\end{proof}
\begin{theo}Assume that $R$ is diagonalizable and all of its eigenvalues are real and positive. If \ $ \Sigma = I$ and $W_i[0] = I$ for all $i \in [L]$ and the step size $\delta$ satisfies
\[ \delta \le {1 \over L} \min\left\{ 1, {1 \over \lambda(R)^{2(L-1)/L}} \right\}, \]
then each $W_i$ converges to $R^{1/L}$ exponentially fast. \end{theo}
\begin{proof} After bringing the update equation (\ref{update_eqn}) into diagonal form, Proposition 2 can be applied to each of the $n$ systems involving the diagonal elements. Since $\delta_c$ in Proposition 2 is monotonically decreasing in $\lambda$, the bound for the maximum eigenvalue of $R$ guarantees linear convergence. 
\vspace{-0.1in} \end{proof}

If the matrix $\Sigma$ is not the identity and  the eigenvectors of $\Sigma$ do not lie in the eigenspaces of $R$, then the update rules for gradient descent remain coupled and the dynamics of the parameters become more complex. Furthermore, if a stochastic gradient method is used, then the update rules may still not decouple even if the input points $\{x_i\}_{i \in [N]}$ are~orthonormal, unless each $x_i$ lies in an eigenspace of $R$.~In this case, taking a batch of points $\{x_j\}_{j \in J}$ which satisfy $\sum_{j \in J} x_j x_j^\top = I$ for some $J \subset [N]$ at each step of the gradient descent simplifies the dynamics.

It was shown in \cite{Moritz} that when the matrices $\{W_i\}_{i\in[L]}$ are close to the identity, the gradient cannot vanish unless $\hat F$ is close to $R$, which can also be seen from (\ref{update_eqn}). However, the stability of the gradient descent algorithm was not addressed. From the proofs for Theorem 1 and Theorem 2, we observe that by keeping all eigenvalues of the matrices $\{W_i\}_{i\in [L]}$ close to each other, it is possible to find a step size that will maintain the stability of the gradient descent while providing an effective convergence rate. 

Note that  if we used distinct unitary matrices instead of the identity, the gradients would still vanish only when $\hat F = R$, and therefore, every local optima would still be the global optima. In addition, gradients with respect to different parameters would likely not become disproportionate since unitary matrices do not amplify or attenuate the eigenvectors of other matrices either. Therefore, using unitary matrices instead of the identity could possibly yield some results comparable to residual networks, although the dynamics of the parameters would be harder to analyze.


\section{Decomposition of Functions}

In the previous section, we have seen that the parametrization of linear residual networks are well suited for  optimization via gradient descent. In this section, we show that a large class of functions can actually be decomposed into two parts each of which can be approximated by a residual network. 

\subsection{Convex Decomposition for Nonlinear Functions}

The following theorem from \cite{Bartlett1}  provides a set of sufficient conditions under which a function can be written as a sequence of functions all of which are close to the identity.

\begin{theo} \cite{Bartlett1} Consider a function $h: \mathbb R^n \to \mathbb R^n$ on a bounded domain. Suppose that it is differentiable, invertible, and $\alpha$-smooth:
\[ \| Dh(y) -  Dh(x) \| \le \alpha \|y - x\| \]
for all $x,y \in \text{dom}(h)$ for some $\alpha > 0$, where $Dh$ is the derivative and $\|Dh(y)\|$ is the operator norm. Further assume that the inverse $h^{-1}$ is Lipschitz, and \hbox{$\text{det}(Dh(x_0))>0$} for some $x_0 \in \text{dom}(h)$. Then for all~$L$, there are $L$ \ functions $h_1, \dots, h_L : \mathbb R^n \to \mathbb R^n$ satisfying 
\[ h_L \circ h_{L-1} \circ \cdots \circ h_1  = h \]
and $\|h_i - \text{Id} \|_L = O(\text{log}L / L)$ for all $i \in [L]$, where $\text{Id}(\cdot)$ is the identity function and $\| \cdot \|_L$ denotes the Lipschitz seminorm. \end{theo}
Theorem 3 provides an existence result in the function space without assuming any fixed structure for the estimator.~If a neural network is used, for example, the width of the network might need to be very large or a certain nonlinearity might be needed at each layer depending on the function to be estimated. 
In the sequel, we show that a residual network that contains rectified linear units (ReLU) as nonlinearities could be used to approximate strictly convex functions.

Note that if $f: D  \to \mathbb R$ is a twice differentiable, strictly convex function over a bounded domain $D \subset \mathbb R^n$, 
 then its gradient $\nabla f$ satisfies all the conditions of Theorem 3. Therefore, we can represent $\nabla f$ with a residual network. 
First consider a univariate function $f: [0, M] \to \mathbb R$ which is continuously differentiable and strictly convex on its domain and has a strictly positive derivative at 0, i.e., $f'(0) > 0$. A first order approximation to $f$ around 0 is
\begin{equation}
\label{firstord}
 \hat f(x) = f(0) + f'(0)x. \end{equation}
 However, given that $f$ is strictly convex, the approximation in (\ref{firstord}) underestimates the function particularly at larger values of $x$. To increase both the estimate and the derivative of the estimate for larger $x$ values, we can instead use
 \[ h_1(x) = x + w_1(x - b_1)_+, \]
 \[ h_2(x) = h_1(x) + w_2(h_1(x) - b_2)_+, \]
 \[ \hat f(x) = f(0) + f'(0)h_2(x),\]
 where $w_1,w_2,b_1,b_2 \in \{z \in \mathbb R : z > 0\}$, and $(z)_+$ denotes $\max\{0,z\}$. The estimate $\hat f$ is strictly increasing, and as $x$ gets large, the derivative of the estimate
  gradually increases to {$(1+w_2)(1+w_1)f'(0)$}, provided that $b_1$ and $b_2$ are not too large. As a result, $\hat f$ provides a better estimate for the original function.
  
 Similarly, given a strictly convex, twice differentiable function $f : D \subset \mathbb R^n \to \mathbb R$ with a domain of the form
 \[ D = [0,M_1] \times \dots \times [0,M_n] \]
 and a gradient with positive coordinates at the origin, 
 we can approximate it with
 \begin{subequations}
 \begin{gather}
\label{layer0} h_0(x) = x,\\
 \hspace{-0.1in} h_i(x) = h_{i-1}(x) + W_i[V_i^\top h_{i-1}(x)-b_i]_+ \ \forall  i \in [L],\\
 \label{layerL} \hat f(x) = c^\top h_L(x) + d,
 \end{gather}
 \end{subequations}
 where all of  $W_i \in \mathbb R^{n \times {m_i}}$, $V_i \in \mathbb R^{n \times {m_i}}$, $b_i \in \mathbb R^{m_i}$ have nonnegative elements, $c \in \mathbb R^{n}$ has positive elements, and $d \in \mathbb R$. The function described by (\ref{layer0})--(\ref{layerL}) is convex because for each $i \in [L]$, every coordinate of $h_i(x)$ is obtained by taking nonnegative combination and pointwise maximum of the coordinates of $h_{i-1}(x)$, both of which are operations that preserve convexity \cite{Boyd}. 
 
 To illustrate how a function described by (\ref{layer0})--(\ref{layerL}) approximates a convex function, consider $\hat f : [0,M_1] \times [0,M_2] \to \mathbb R$ defined as
 \begin{subequations}
 \begin{gather}
\label{f11} h_1(x) = x +
 \left(x - b\right)_+\\
\label{f22}  \hat f(x) = [1 \quad 1]\ h_1(x)
 \end{gather}
 \end{subequations}
 where $W_1=V_1=I$, $b = [b_1 \ \ b_2]^\top$ and $b_1, b_2 > 0$. Gradient of $\hat f$ near the origin is $[1\ \ 1]^\top$, and when $[1 \ \ 0]x - b_1 \ge 0$ or $[0\ \ 1]x - b_2 \ge 0$ holds, the gradient gets an increment of $[1\ \ 0]^\top$ or $[0\ \ 1]^\top$, respectively. Figure 2 shows the level curves of the function obtained. 

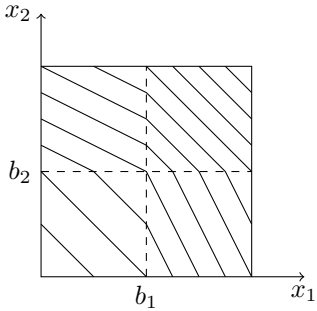
\begin{figure}[h]
\label{bowl1}
\centering
\begin{tikzpicture}[scale=0.7]
\draw[<->] (5,0) node[below]{$x_1$} -- (0,0) -- (0,5) node[left]{$x_2$};
\draw (2,0) node[below]{$b_1$};
\draw (0,2) node[left]{$b_2$};
\draw (4,0)--(4,4)--(0,4);
\draw[dashed] (2,0)--(2,4);
\draw[dashed] (0,2)--(4,2);
\draw (1,0)--(0,1);
\draw (2,0)--(0,2);
\draw (0,2.5)--(1,2)--(2,1)--(2.5,0);
\draw (0,3)--(2,2)--(3,0);
\draw (0,3.5)--(2,2.5)--(2.5,2)--(3.5,0);
\draw (0,4)--(2,3)--(3,2)--(4,0);
\draw (1,4)--(2,3.5)--(3.5,2)--(4,1);
\draw (2,4)--(4,2);
\draw (2.5,4)--(4,2.5);
\draw (3,4)--(4,3);
\draw (3.5,4)--(4,3.5);
\end{tikzpicture}
\caption{Level curves of $\hat f$ defined in (\ref{f11})--(\ref{f22}) coincide with a bowl-shaped function.}
\end{figure}
 
 Note that the level curves shown in Figure 2 resemble the level curves of a bowl-shaped function. If we used a different matrix $V$ instead of the identity matrix, we would see that the columns of $V$  determine the normals of the lines beyond which the gradient is incremented, while the elements of $b$ determine the distance of these lines to the origin.
 
 Even though the sequence of functions given in (\ref{layer0})--(\ref{layerL}) has been shown to represent only convex functions, it can be used as the building block to represent a much broader class of functions. To show this, let $f:  [0,M_1] \times \dots \times [0,M_n] \to \mathbb R^m$ be a twice-differentiable function, and let $f_k(x) \in \mathbb R$ denote the $k^\text{th}$ coordinate of $f(x)$. Then $f_k$ can be written~as
\[ f_k(x) = r(x) - s(x),\]
where both $r$ and $s$ are strictly convex and their gradients at the origin have strictly positive coordinates, e.g., 
\[ r(x) = {\alpha \over 2} x^\top x + \beta^\top x +  f_k(x),\]
\[s(x) =  {\alpha \over 2} x^\top x  + \beta^\top x,\]
where $\beta \in \mathbb R^n$ is a vector with positive elements and $\alpha \in \mathbb R$ is a positive constant large enough to make the Hessian of $r(x)$ positive definite everywhere in the domain of $f_k$. Then, a pair of residual networks (\ref{layer0})--(\ref{layerL}) can be used to approximate each coordinate of $f$, and consequently, $m$ pairs of residual networks can be used to approximate~$f$.
%
This architecture is tested on the MNIST data set in {{Section~5.}}

Given that a single residual network has been shown to perform well in practice, one could question the necessity of decomposing the functions into two parts.~A similar decomposition for linear mappings is shown in Section \ref{poseigen} to be necessary and to improve the convergence of the parameters. 

\subsection{Positive Eigenvalue Decomposition for Linear Functions}
\label{poseigen}

In Section 2, Theorem 2 was stated only for the linear mappings with positive eigenvalues.~Even though we could argue that we might as well initialize some of the diagonal elements of the weight matrices with -1 to estimate negative eigenvalues, this is not possible without knowing the signs of the eigenvalues a priori. On the other hand, if all the weight matrices are initialized as the identity, then the diagonal elements corresponding to the negative eigenvalues converge to 0, not to a negative value, which is shown next.  

\begin{prop} Assume that $\lambda < 0$ and $w_i[0] = 1$ is used for all $i \in [L]$ to initialize the gradient descent algorithm to solve
\[ \min_{(w_1, \dots, w_L) \in \mathbb R^L} {1\over 2N} \sum_{i=1}^N \left(w_L \dots w_2w_1 x_i - \lambda x_i \right)^2. \]
Then, each $w_i$ converges to 0 unless $\delta > {1 \over \sigma (1- \lambda)}$, where $\sigma = {1 \over N} \sum_{i=1}^N{x_i^2}$. \end{prop}

\begin{proof} Similar to the proof of Proposition 2, we can write the update rule for any weight $w_i$ as
\begin{equation*}
 w[k+1] 
 =  w[k] \left( 1 - \delta \sigma w^{L-2}[k] \left( w^L[k] - \lambda \right) \right)
 \end{equation*}
which has one equilibrium at $w^* = \lambda^{1/L}$ and another at  $w^* = 0$. 
If $0 < \delta \le {1/ \sigma(1-\lambda)}$ and $w[0] = 1$, it can be shown by induction that
\[ 0 \le 1 - \delta \sigma w^{L-2}[k]\left(w^L[k] - \lambda \right) < 1\]
for all $k \ge 0$. As a result, $w[k]$ converges to 0. \end{proof}

To resolve this issue of convergence for negative parameters, we can decompose the linear mappings. First consider
\[ \min {1 \over 2N} \sum_{i=1}^N \left(w_L \dots w_1 x_i - z_L \dots z_1 x_i - \lambda x_i \right)^2 \]
where minimization is over $\{w_i, z_i\}_{i \in [L]}$ and  $w_i, z_i \in \mathbb R$ for all $i \in [L]$. Denoting $(\prod w_i[k] - \prod z_i[k] - \lambda)$ by $e[k]$,
we can write the gradient update rule for $w_i[k]$ and $z_i[k]$ as
\[ w_i[k+1] = w_i[k] - \delta \sigma e[k]\prod_{j \neq i}w_j[k],\]
\[ z_i[k+1] = z_i[k] + \delta \sigma e[k] \prod_{j \neq i} z_j[k].\]
If $w_i[0] = z_i[0] = 1$ for all $i \in [L]$, we obtain
\begin{subequations}
\begin{gather}
 w[k+1] = w[k] - \delta \sigma w^{L-1}[k]e[k],\\
 z[k+1] = z[k] + \delta \sigma z^{L-1}[k] e[k], 
 \end{gather}
\end{subequations}
where $e[0] = - \lambda$. Note that if $\lambda < 0$, $w[k]$ decreases and $z[k]$ increases initially, bringing $e[k]$ closer to zero. Even though the origin $w^* = z^*  = 0$ is still an equilibrium, it is unstable and $(w[k],z[k])$ cannot converge to it.

Based on the scalar case, we can build a double linear network to estimate a linear mapping in higher dimensions as well:
\[ \min {1 \over 2N} \sum_{i=1}^N \ \left\| (W_{L} \cdots W_1 - Z_{L} \cdots Z_1)x_i - R x_i \right\|^2, \]
where $R, W_i, Z_i \in \mathbb R^{n \times n}$ for all $i \in [L]$, and the minimization is over the matrices $\{W_i, Z_i\}_{i \in [L]}$. If the gradient descent algorithm is initialized with the identity matrices, we expect  the convergence of the training error for the double network to be faster than that for the single network. To confirm this, we generated a set of random diagonalizable matrices in  $\mathbb R^{20 \times 20}$ with random eigenvectors drawn from the normal distribution $\mathcal N(0, I)$ and random eigenvalues drawn from the uniform distribution on $[-1.5, 1.5]$. We compared the training errors for both networks by choosing $L=20 $.
 As expected, the convergence rate of the double network was consistently better than that of the single network for all the matrices generated.~It was also observed that the gradient descent for the single network became unstable  for some of the matrices, while the algorithm remained stable for the double network. Figure 3 shows a typical~comparison of the training error for the two networks.

\begin{figure}[h]
\label{doublenet}
\hspace{-0.08in}
\includegraphics[scale=0.59]{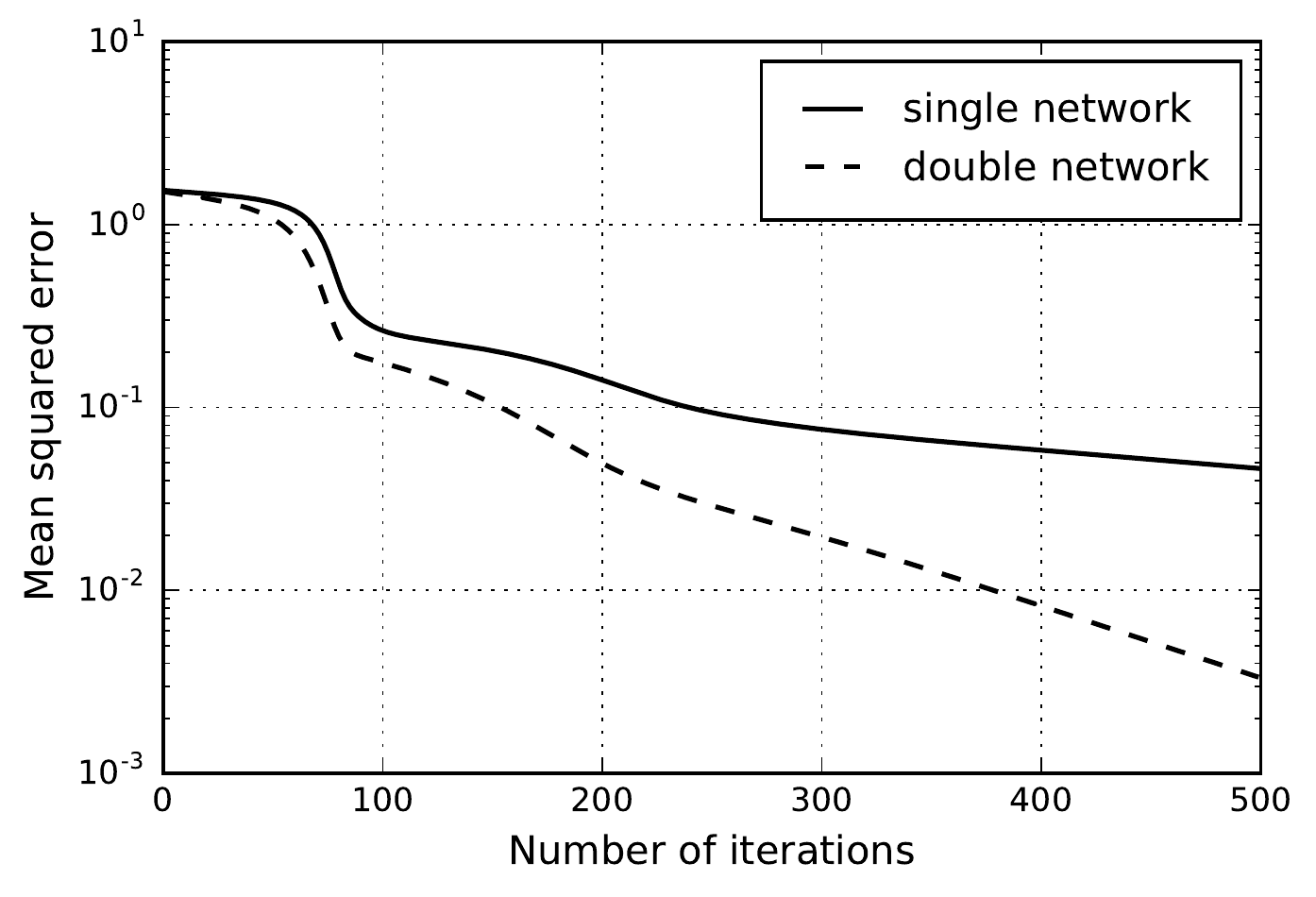}
\vspace{-0.2in}
\caption{Comparison of training a single residual network and a double residual network.}
\end{figure}

\section{Properties of the Local Optima}

In the previous section, we showed that a pair of residual networks with ReLU nonlinearities could be used to approximate a large class of functions. In this section, we show that the solutions obtained by training the network described in (\ref{layer0})--(\ref{layerL}) via gradient descent have low Lipschitz constants and are unlikely to overfit the data. 

Consider a simplified version of the network given in (\ref{layer0})--(\ref{layerL}) with $d = 0$ and $V_i = I$ for all $i \in [L]$:
\begin{subequations}
\begin{gather}
\label{net0} h_0(x) = x,\\
h_i(x) = h_{i-1}(x) + W_i\left[ h_{i-1}(x) - b_i \right]_+ \ \forall i \in [L], \\
\label{netL} \hat f(x) = c^\top x,
\end{gather}
\end{subequations}
where $c \in \mathbb R^n$ has strictly positive entries, and $b_i \in \mathbb R^n$, $W_i \in \mathbb R^{n \times n}$  have nonnegative entries for all $i\in [L]$. 

Given a set of points $\{x_i\}_{i \in [N]}$ and their labels $\{y_i\}_{i\in [N]}$, assume that we train a network described with (\ref{net0})--(\ref{netL}) by minimizing the mean squared error:
\begin{equation}
\label{prob12}
 \min_{c, \{W_j, b_j\}} {1 \over 2} \sum_{i=1}^N \big( \hat f(x_i) - y_i \big)^2.\end{equation}
Let $\{W_i^*, b_i^*\}_{i \in [L]}$ and $c^*$ denote the parameters obtained as the local optimum of the problem (\ref{prob12}). For any $x_i$ at which the function $\hat f$ is differentiable, we can write
\[ \hat f^*(x_i) = {c_l^i}^\top \left[ h^*_{l-1}(x_i) + W_l^*(h^*_{l-1}(x_i) - b_l^*)_+ \right] + d_l^i \]
where
\[ {c_l^i}^\top = {c^*}^\top(I + \tilde W_n^i) \cdots (I + \tilde W_{l+1}^i) \]
and $\tilde W_k^i$ satisfies
\[ 0 \le \tilde W_k^i \le W_k^*. \]
The matrix $\tilde W_k^i$ corresponds to the change in the gradient caused by the activated ReLU functions at layer $k$ for the point $x_i$, and consequently, $c_l^i$ denotes the gradient of the estimate $\hat f^*$ with respect to $h^*_{l}$ at point $x_i$. Given that all elements of $c^*$ are strictly positive, the vector $c_l^i$ also has strictly positive elements for all $i \in [N]$ and $l \in [L]$.

The first order local optimality condition for $\{W_l^*\}_{l \in [L]}$ and $c^*$ dictates that
\begin{subequations}
\begin{gather}
 \hspace{-0.1in}\sum_{i=1}^N c_l^i (\hat f^*(x_i) - y_i)(h_{l-1}^*(x_i) - b_l^*)^\top_+ = 0  \ \forall l \in [L], \hspace{-0.1in} \label{cond111} \\
 \sum_{i=1}^N (\hat f^*(x_i) - y_i)h^*_L(x_i)^\top = 0. \label{cond222} \end{gather}
\end{subequations}
Note that the conditions (\ref{cond111})--(\ref{cond222}) impose that the error vector
\[ [(\hat f^*(x_1)-y_1) \ \ (\hat f^*(x_2) - y_2) \ \ \cdots \ \ (\hat f^*(x_N) - y_N) ]^\top \]
is orthogonal to $ln^2 + n$ vectors. Though these vectors are not necessarily linearly independent, the same indices of these vectors are zero for the points on the same affine piece of the function $\hat f^*$ since the ReLU functions
\[ (h_{l-1}^*(x_i) - b_l^*)_+\]
are activated and deactivated simultaneously for these points. As a result, each affine piece of $\hat f^*$ is likely to be a solution to a weighted-least-squares problem for the points corresponding to that piece.

If the bias parameters $\{b_l^*\}_{l \in [L]}$ are distributed such that $\mathbb R^N$ is spanned by the $ln^2 + n$ vectors that the error vector is orthogonal to,
  the estimator fits all the data points perfectly. However, this is unlikely to happen due to the optimization of the bias parameters. The gradient of the cost function (\ref{prob12}) with respect to $b_l$ around a local optimum $\{W^*_i, b^*_i\}_{i \in [L]}$ is
\[ -\sum_{i=1}^N \text{diag}({\mathbf 1}\left\{h^*_{l-1}(x_i) - b_l^* \ge 0 \right\}){W_l^*}^\top c_l^i (\hat f^*(x_i) - y_i), \]
where $\mathbf 1\{ z \}$ is an indicator function. If the number of layers of the network is large, $W_l^*$ is expected to be very close to 0 for the residual network. In addition, the points for which the ReLU function is inactive do not contribute to the gradient. As a result, the gradient of the cost function with respect to the bias parameters is likely to vanish quickly, and consequently, the bias parameters change very little during training. This suggests that the final values of the bias parameters heavily depend on their initialization. 

To verify these claims, we trained two estimators $\hat f_1, \hat f_2 : [0,1] \to \mathbb R$ with identical architectures to approximate a piecewise affine function $f : [0,1] \to \mathbb R$ whose derivative is
\[ f'(x) = \left\{ \begin{array}{c l} 1 & \text{ if } x \in (0,0.3)\cup (0.5,0.7), \\ -2 & \text{ if } x \in (0.3, 0.5), \\ 
-1 & \text{ if } x \in (0.7, 1). \end{array} \right. \]
The estimator  $\hat f_1$, and similarly $\hat f_2$, were built as $$\hat f_1  = \hat f_{11} - \hat f_{12} + d_1,$$
 where $d_1 \in \mathbb R$ and $\hat f_{11}, \hat f_{12} : [0,1] \to \mathbb R$ were as described in (\ref{net0})--(\ref{netL}) with 10 layers and scalar weight parameters. The initial values of $\{b_i\}$ were drawn from the uniform distribution on $[0, 0.5]$ and $[0,1]$ for the estimators $\hat f_1$ and $\hat f_2$, respectively. Figure \ref{imperfectfit} shows the function $f$ and the estimates $\hat f_1$ and $\hat f_2$ obtained by Nesterov's accelerated gradient descent algorithm \cite{Nest}. 

\begin{figure}[h]
\label{imperfectfit}
\includegraphics[scale=0.58]{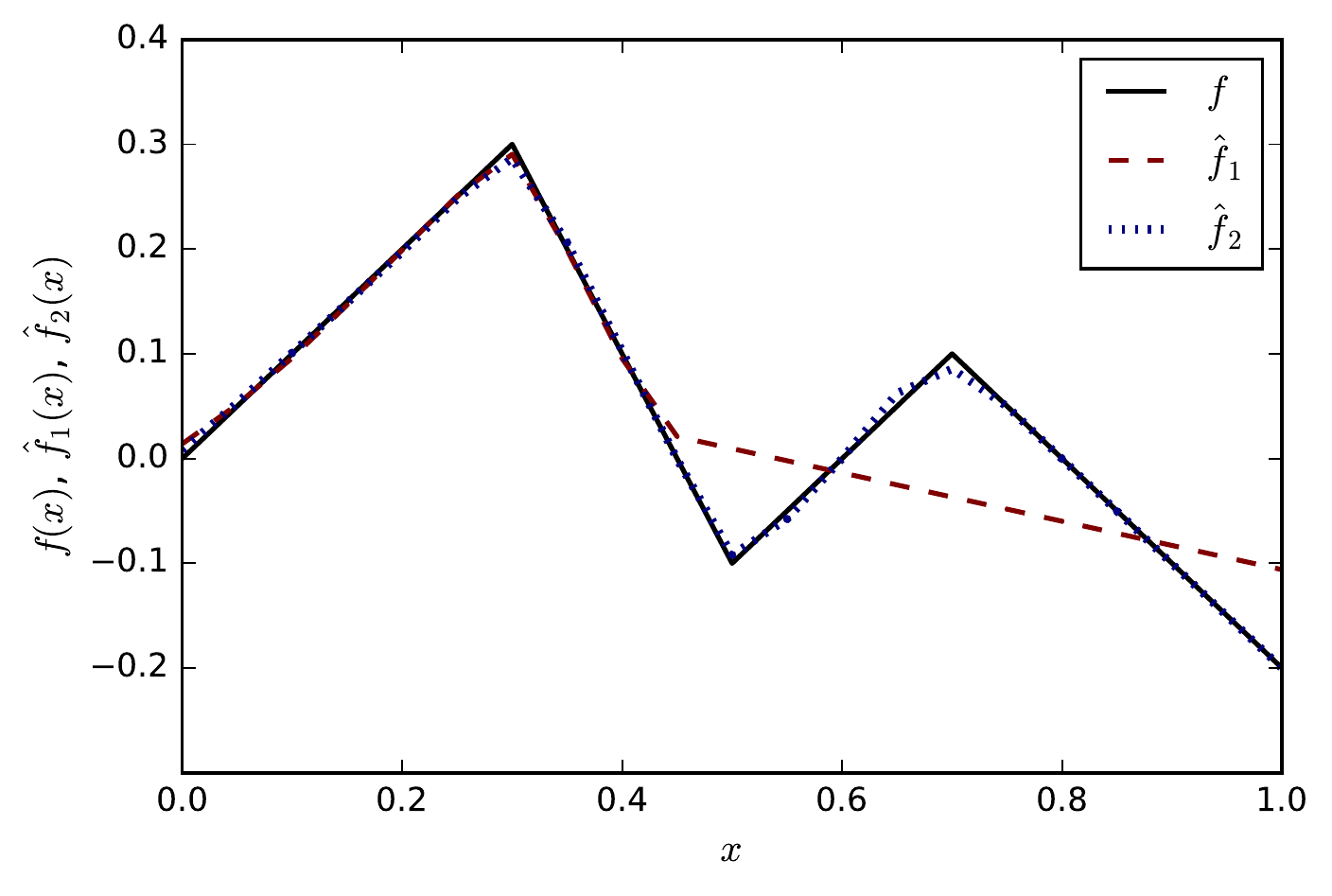}
\vspace{-0.2in}
\caption{The estimates $\hat f_1$ and $\hat f_2$ have the same network architecture, but different initialization for the bias parameters. Depending only on the initialization of the bias parameters, the estimate $\hat f_1$ fails to fit the data perfectly, in which case it provides an estimate with a low Lipschitz constant.}
\end{figure}

Initialized with a larger range for the bias parameters, $\hat f_2$ fits all the data points. The estimate $\hat f_1$, on the other hand, fails to fit the data perfectly even though the network had 10 layers and it could have had 10 segments. Nevertheless, over the region where it fails to fit the data, it is close to a linear estimate of the points belonging to that region. Consequently, the Lipschitz constant of the estimate at a certain point is either the same as or less than that of the original function. 

Small Lipschitz constant, and correspondingly, small spectral norm of an estimator is an indicator of its low excess risk \cite{Bartlett2}. Therefore, we expect an estimator with the decomposed structure to generalize well, which is confirmed on the MNIST data set in the next section.

\section{Experiment on MNIST}

We tested the decomposed  model introduced in Section 3.1 on the MNIST data set, which contains images of handwritten figures $\{0, 1, \dots, 9\}$. Since there are 10 classes, we constructed 10 pairs of residual networks in total. Instead of feeding the raw images into these networks, we first used one convolutional layer with 64 filters of size ${6 \times 6}$ to extract the edges as the features, and then reduced the dimension of the output of this layer by taking the maximum of every non-overlaping ${4 \times 4}$ window. The output of this layer was then given as the input to all of the residual networks, each of which had 3 layers. 

We trained this network for 12 epochs and recorded its accuracy on the training and the test data after each epoch, which is plotted in Figure \ref{mnistfig}. The number of parameters was much larger than standard networks since we used 20 residual networks in total. Furthermore, no explicit regularization or methods such as drop-out or batch-normalization was used. Nevertheless, the training and the test errors were remarkably close to each other throughout the training. After the 12th epoch, the training and the test accuracy were 97.91\% and 97.58\%, respectively. 

\begin{figure}[h]
\label{mnistfig}
\includegraphics[scale=0.6]{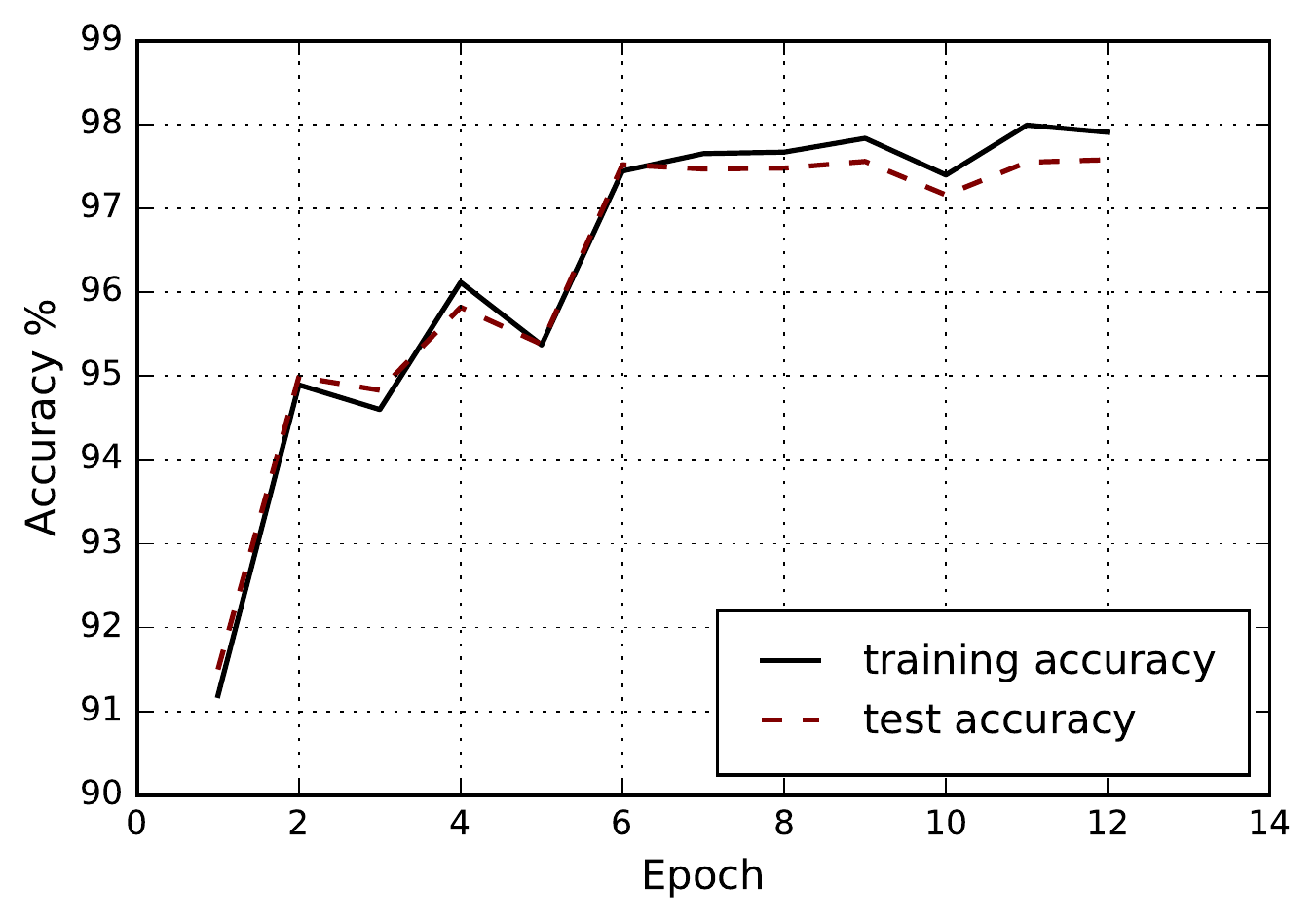}
\vspace{-0.2in}
\caption{Accuracy of the the decomposed network on the training and the test data with no explicit regularization, batch-normalization or drop-out.}
\end{figure}

%

\section{Discussion}

Analyzing the dynamics of the optimization algorithms is crucial for a complete understanding of deep neural networks. Step size of the gradient descent, for example, is a key factor for the Lyapunov stability of its equilibria and should not be disregarded: it determines whether the algorithm can converge to the local optima. Taking the effect of the step size into account, we
showed that~most of the local optima of deep linear networks actually cannot be discovered via gradient-based methods, even though all of them were known to be a global optimum \cite{Kawaguchi16}.

Similarly, the equilibria of other networks could also be unstable for a given step size, and the gradient-based algorithms might only get close to a local minimum but not converge to it \emph{even if the algorithm is not stochastic}. Consequently, an inexact but approximate solution is obtained for the optimization problem, and this naturally contributes to the lack of overfitting in deep neural networks despite their large number of parameters \cite{Recht}.

We also observed that the cost function used for training~a residual network can easily become insensitive to the bias parameters, and the final values of these parameters might heavily depend on their initialization. Though this seems like a problem, it is another factor contributing to the generalization. This happens only for the bias parameters in~the residual networks, but it could happen to the weight parameters as well in other types of networks, and that is what we already know as {the vanishing gradient problem}.
In this respect,~the hardness of the optimization provides some level of regularization for deep neural networks.

We showed that the parametrization of residual networks allows the equilibria of the gradient descent algorithm to remain stable, thereby facilitating the optimization. If unitary matrices are used instead of the identity, similar results could possibly be obtained.
In addition, using an orthonormal set of data points at each step of the gradient descent was seen to help decouple the dynamics of the parameters. This might be a partial explanation for the improvements provided by using batch-normalization in practice \cite{Batchnorm}.

We proposed a network architecture which provides an understanding of how each layer improves the approximation in a deep neural network. We chose convexity as a property to decompose the functions into two parts. The fact~that the function in each layer remained invertible with this decomposition was critical.
Other decompositions could alternatively be generated and tested.

We showed that the architecture introduced generalizes very well on the MNIST data set. The training, however, was very slow due the large number of parameters and the small gradients of the bias parameters. The convergence could possibly be improved  by using alternatives to  gradient-based algorithms to update the bias parameters, although this might risk overfitting the~data. 

It was known that deep linear networks  produce solutions with small Lipschitz constants under some conditions \cite{Gunasekar}, and the Lipschitz constant of the estimators could be used to explain their generalization \cite{Bartlett2}. We demonstrated that the solutions obtained with residual networks also have the same property, and hence, generalize well. 

Lastly, enlarging the region of attraction of the equilibria by choosing a specific control is a standard problem in nonlinear control theory. Finding a state dependent step size to improve the convergence of the gradient descent for neural networks is an ongoing work.

\nocite{langley00}

\bibliography{example_paper}
\bibliographystyle{icml2018}

\end{document}